\documentclass{article}

\usepackage[preprint]{neurips_2022}

\usepackage[utf8]{inputenc} 
\usepackage[T1]{fontenc}    
\usepackage{hyperref}       
\usepackage{url}            
\usepackage{booktabs}       
\usepackage{amsfonts}       
\usepackage{nicefrac}       
\usepackage{microtype}      
\usepackage{xcolor}         
\usepackage{enumitem}
\usepackage{graphicx}

\usepackage{amsmath}

\DeclareMathOperator*{\argmin}{arg\,min}

\usepackage{amsthm}
\theoremstyle{plain}
\newtheorem{theorem}{Theorem}[section]
\newtheorem{proposition}[theorem]{Proposition}

\theoremstyle{definition}

\title{Annealed Training for Combinatorial \\ Optimization on Graphs}

\author{%
   Haoran Sun \\
   Georgia Tech \\
   \texttt{hsun349@gatech.edu} \\
   \And
   Etash K. Guha \\
   Georgia Tech \\
   \texttt{etash@gatech.edu} \\
   \And
   Hanjun Dai \\
   Google Brain \\
   \texttt{hadai@google.com} \\
}

\begin{document}

\maketitle

\begin{abstract}
The hardness of combinatorial optimization (CO) problems hinders collecting solutions for supervised learning. However, learning neural networks for CO problems is notoriously difficult in lack of the labeled data as the training is easily trapped at local optima. In this work, we propose a simple but effective annealed training framework for CO problems. In particular, we transform CO problems into unbiased energy-based models (EBMs). We carefully selecting the penalties terms so as to make the EBMs as smooth as possible. Then we train graph neural networks to approximate the EBMs. To prevent the training from being stuck at local optima near the initialization, we introduce an annealed loss function.
An experimental evaluation demonstrates that our annealed training framework obtains substantial improvements. In four types of CO problems, our method achieves performance substantially better than other unsupervised neural methods on both synthetic and real-world graphs.
\end{abstract}

\section{Introduction}
Combinatorial Optimization (CO) problems occur whenever there is a requirement to select the best option from a finite set of alternatives. They arise in a wide range of application areas, including business, medicine, and engineering \citep{paschos2013applications}.
Many CO problems are NP-complete \citep{karp1972reducibility, garey1979computers}. 
Thus, excluding the use of exact algorithms to find the optimal solution \citep{padberg1991branch, wolsey1999integer}, different heuristic methods are employed to find suitable solutions in a reasonable amount of time \citep{nemhauser1978analysis, dorigo2006ant, hopfield1985neural, kirkpatrick1983optimization}.

Often, instances from the same combinatorial optimization problem family are solved repeatedly, giving rise to the opportunity for learning to improve the heuristic~\citep{bengio2020machine}. Recently, learning algorithms for CO problems has shown much promise, including supervised \citep{khalil2016learning, gasse2019exact, li2018combinatorial, selsam2018learning, nair2020solving}, unsupervised \citep{karalias2020erdos, toenshoff2021graph}, and reinforcement learning \citep{dai2017learning, sun2020improving, yolcu2019learning, chen2019learning} 
The success of supervised learning relies on labeled data. However, solving a hard problem could take several hours or even days and is computationally prohibitive \citep{yehuda2020s}. Reinforcement learning, suffering from its larger state space and lack of full differentiability, tends to be more challenging and more time-consuming to train.

Unsupervised learning usually transforms a CO problem into an optimization problem with a differentiable objective function $f$ where the minima represent discrete solutions \citep{hopfield1985neural, smith1999neural, karalias2020erdos}.
Although this framework allows for efficient learning on large, unlabeled datasets, it is not without challenges. 
In the absence of the labels, the objective function is typically highly non-convex \citep{mezard2009information}. During learning, especially for neural networks, the model's parameters can easily get trapped near a local optimum close to the initialization, never reaching the optimal set of parameters. Such a phenomenon makes unsupervised learning for CO problems extremely hard.

To address this challenge, we propose an annealed training framework. 
In detail, given a CO problem, we consider a tempered EBM $P_\tau \propto e^{-f(\mathbf{x})/\tau}$, where the energy function $f$ 
unifies constrained or unconstrained CO problems via the big-M method, that's to say, adding large penalties for violated constraints. We derive the minimum values of the penalty coefficient in different CO problems that gives us the smoothest, unbiased energy-based models. We train a graph neural network (GNN) that predicts a variational distribution $Q_\phi$ to approximate the energy-based model $P_\tau$. During training, we set a high initial temperature $\tau$ and decrease it gradually during the training process. When $\tau$ is large, $P_\tau$ is close to a uniform distribution and only has shallow local optima, such that the parameter $\theta$ can traverse to distant regions. When $\tau$ decreases to values small enough, the unbiased model $P_\tau$ will concentrate on the optimal solutions to the original CO problem.

The experiments are evaluated on four NP-hard graph CO problems: maximum independent set, maximum clique, minimum dominate set, and minimum cut. On both synthetic and real-world graphs, our annealed training framework achieves excellent performance compared to other unsupervised neural methods \citep{toenshoff2021graph, karalias2020erdos}, classical algorithms \citep{aarts2003local, bilbro1988optimization}, and integer solvers \citep{gurobi19gurobi}. The ablation study demonstrates the importance of selecting proper penalty coefficients and cooling schedules.

In summary, our work has the following contributions:
\begin{itemize}[leftmargin=*,nosep,nolistsep]
\item We propose an annealed learning framework for generic unsupervised learning on combinatorial optimization problems. It is simple to implement, yet effective to improve the unsupervised learning across various problems on both synthetic and real graphs.
\item We conducted ablation studies that show: 1) annealed training enables the parameters to escape from local optima and traverse a longer distance, 2) selecting proper penalty coefficients is important, 3) Using initial temperature large enough is important.
\end{itemize}

\section{Annealed Training for Combinatorial Optimization}
We want to learn a graph neural network $G_\theta$ to solve combinatorial optimization problems. Given an instance $I$, the $G_\theta$ generates a feature $\phi = G_\theta(I)$ that determines a variational distribution $Q_\phi$, from which we decode solutions.
This section presents our annealed training framework for training $G_\theta$. We first show how to represent CO problems via an energy-based model. Then, we define the annealed loss function and explain how it helps in training. Finally, we give a toy example to help the understanding.

\subsection{Energy Based Model}
We denote the set of combinatorial optimization problems as $\mathcal{I}$. An instance $I \in \mathcal{I}$ is
\begin{equation}
    I = \left(c(\cdot), \{\psi_i\}_{i=1}^m\right) := \argmin_{\mathbf{x}\in \{0, 1\}^n} c(\mathbf{x}) \quad \text{s.t. } \psi_i(\mathbf{x}) = 0, \quad i=1, ..., m \label{eq:cons}
\end{equation}
where $c(\cdot)$ is the objective function we want to minimize and $\psi_i \in \{0, 1\}$ indicates whether the i-th constraint is satisfied or not.
We first rewrite the constrained problem into an equivalent unconstrained form via big M method: 
\begin{equation}
    \label{eq:energy}
    \argmin_{\mathbf{x} \in \{0, 1\}^n} f^{(I)}(\mathbf{x}) := c(\mathbf{x}) + \sum_{i=1}^m \beta_i \psi_i(\mathbf{x}), \quad \beta_i \ge 0
\end{equation}
If $f^{(I)}$ has its smallest values on optimal solutions for \eqref{eq:cons}, we refer it to unbiased. The selection of penalty coefficient $\beta$ plays an important role in the success of training and we will discuss our choice of $\beta$ detailedly in section \ref{sec:cases}.
Using unbiased $f^{(I)}$ to measure the fitness of a solution $x$, we can define the unbiased energy-based models (EBMs):
\begin{equation}
    \label{eq:model}
    P^{(I)}_\tau(\mathbf{x}) \propto e^{-f^{(I)}(\mathbf{x})/\tau}
\end{equation}
The EBMs naturally introduce a temperature $\tau$ to control the smoothness of the distribution. When $f$ is unbiased, it has the following property:
\begin{proposition}
\label{prop:temp}
Assume $f$ is unbiased, that's to say, all minimizers of \eqref{eq:energy} are feasible solutions for \eqref{eq:cons}. When the temperature $\tau$ increases to infinity, the energy based model $P_\tau$ converges to a uniform distribution over the whole state space $\{0, 1\}^n$. When the temperature $\tau$ decreases to zero, the energy based model $P_\tau$ converges to a uniform distribution over the optimal solutions for \eqref{eq:cons}.
\end{proposition}
The proposition above shows that the temperature $\tau$ in unbiased EBMs provides an interpolation between a flat uniform distribution and a sharp distribution concentrated on optimal solutions. This idea is the key to the success of simulated annealing \citep{kirkpatrick1983optimization} in inference tasks. We will show that the temperature also helps in learning.

\subsection{Tempered Loss and Parameterization}
We want to learn a graph neural network $G_\theta$ parameterized by $\theta$. Given an instance $I \in \mathcal{I}$, $G_\theta(I) = \phi$ generates a vector $\phi$ that determines a variational distribution $Q_\phi$ to approximate the target distribution $P_\tau^{(I)}$. We want to minimize the KL-divergence:
\begin{align}
    D_\text{KL}(Q_\phi||P_\tau^{(I)}) &= \int Q_\phi(\mathbf{x}) \Big(\log Q_\phi(\mathbf{x})
     - \log \frac{e^{-f^{(I)}(\mathbf{x})/\tau}}{\sum_{\mathbf{x}\in\{0, 1\}^n} e^{-f^{(I)}(\mathbf{x})/\tau}} \Big) d\mathbf{x} \\
    &=\ \frac{1}{\tau} \mathbb{E}_{\mathbf{x}\sim Q_\phi(\cdot)} [f^{(I)}(\mathbf{x})] - H(Q_\phi) 
  + \log \sum_{\mathbf{x}\in\{0, 1\}^n} e^{-f^{(I)}(\mathbf{x})/\tau}
    \label{eq:kl}
\end{align}
Remove the terms not involving $\phi$ and multiply the constant $\tau$, we define our annealed loss functions for $\phi$ and $\tau$ as:
\begin{align}
    L_\tau(\phi, I) &= \mathbb{E}_{\mathbf{x} \sim Q_{\phi} (\cdot)}[f^{(I)} (\mathbf{x})] - \tau H(Q_\phi) \label{eq:phi} \\ 
    L_\tau(\theta) &= \mathbb{E}_{I\sim \mathcal{I}}\left[\mathbb{E}_{\mathbf{x} \sim Q_{G_\theta(I)} (\cdot)}[f^{(I)} (\mathbf{x})] - \tau H(Q_{G_\theta(I)} )\right] \label{eq:theta}
\end{align}
In this work, we consider the variational distribution as a product distribution:
\begin{equation}
    Q_\phi(x) = \prod_{i=1}^n (1 - \phi_i)^{1 - x_i} \phi_i^{x_i} \label{eq:q}
\end{equation}
where $\phi \in [0, 1]^n$. Such a form is a popular choice in learning graphical neural networks for combinatorial optimization \citep{li2018combinatorial, dai2020framework, karalias2020erdos} for its simplicity and effectiveness.
However, direct applying it on unsupervised learning is challenging. Different from supervised learning, where the loss function cross entropy is convex for $\phi$, $L_\tau(\phi, I)$ in unsupervised learning could be highly non-convex , especially when $\tau$ is small. 

\subsection{Annealed Training}
To address the non-convexity in training, we employ an annealed training. In particular, we use a large initial temperature $\tau_0$ to smooth the loss function and reduce $\tau_t$ gradually to zero during training. From proposition \ref{prop:temp}, it can be seen as a curriculum learning \citep{bengio2009curriculum} along the interpolation path from the easier uniform distribution towards harder target distribution.

Why is it helpful? We need a more detailed investigation of the training procedure to answer this question. Since the loss function \eqref{eq:theta} is the expectation over the set of instances $\mathcal{I}$, we typically use a batch of instances $I_1, ..., I_B$ to calculate the empirical loss $\hat{L}_\tau(\theta)$ and perform stochastic gradient descent. It gives:
\begin{align}
    \nabla_\theta \hat{L}_\tau(\theta) 
    &= \sum_{i=1}^B \nabla_\theta L_\tau(G_\theta(I_i), I_i) \\
    &= \sum_{i=1}^B \frac{\partial G_\theta(I_i)}{\partial \theta} \nabla_\phi L_\tau(\phi, I_i)|_{\phi = G_\theta(I_i)} \\
    &= \mathbb{E}_{I\sim \mathcal{I}}\left[ \frac{\partial G_\theta(I)}{\partial \theta} \nabla_\phi L_\tau(\phi, I)|_{\phi = G_\theta(I)} \right] + \xi \label{eq:xi}\\
    &\approx \mathbb{E}_{I\sim \mathcal{I}}\left[ \frac{\partial G_\theta(I)}{\partial \theta} (\nabla_\phi L_\tau(\phi, I)|_{\phi = G_\theta(I)} + \zeta) \right] \label{eq:zeta}
\end{align}
In \eqref{eq:xi}, we assume the batch introduces a stochastic term $\xi$ in gradient w.r.t. $\theta$. In \eqref{eq:zeta}, we incorporate the stochastic term into the gradient with respect to $\phi$. When we assume $\zeta$ is a Gaussian noise, the inner term $g = \nabla_\phi L_\tau(\phi, I)|_{\phi = G_\theta(I)} + \zeta$ performs as a stochastic Langevin gradient with respect to $\phi$ \cite{welling2011bayesian}. Since the training data is sampled from a fixed distribution $I \sim \mathcal{I}$, the scale of the noise $\zeta$ is also fixed. When $L_\tau(\phi, i)$ is unsmooth, the randomness from $\zeta$ is negligible compared to the gradient $\nabla L_\tau(\phi, i)$ and can not bring $\phi$ out of local optima. By introducing the temperate $\tau$, we smooth the loss function and reduce the magnitude of $\nabla L_\tau(\phi, i)$. The annealed training performs an implicit simulated annealing \citep{kirkpatrick1983optimization} for $\phi$ during the training.

\subsection{A Toy Example}
To have a more intuitive understanding of the annealed training, we look at a toy example. Consider a maximum independent set problem on an undirected, unweighted graph $G=(V, E)$, the corresponding energy function $f(x)$ is: 
\begin{equation}
    f(x) = - \sum_{i=1}^{n} x_i + \sum_{(i, j) \in E} x_i x_j
\end{equation}
Its correctness can be justified by proposition \ref{prop:mis}. When we use the variational distribution $Q_\phi$ in \eqref{eq:q}, the first term in $L_\tau(\phi, I)$ becomes to:
\begin{equation}
    \mathbb{E}_{\mathbf{x} \sim Q_{\phi} (\cdot)}[f^{(I)} (\mathbf{x})] = - \sum_{i=1}^n \phi_i + \sum_{(i, j) \in E} \phi_i \phi_j
\end{equation}
and accordingly, the gradient w.r.t. $\phi$ is:
\begin{equation}
    g = - 1 + 2 \sum_{j \in N(i)} \phi_j + \tau (\log \phi_i  - \log (1 - \phi_i)) + \zeta
\end{equation}
where we assume $\zeta \sim \mathcal{N}(0, \sigma^2)$ for a very small $\sigma$. When the temperature $\tau = 0$, $\phi_i$ will collapse to either $0$ or $1$ very fast. When $\phi_i = 1$, we have $g = -1 + \zeta$, when $\phi_i = 0$, we have $g \ge 1 + \zeta$. Since $\sigma$ is small, the noise $\zeta$ can hardly has an effect and $\phi$ will be stuck at local optima, i.e. any maximal independent set such as figure.~\ref{fig:ex} (a). In figure.~\ref{fig:ex}, we simulate the input (a) at decreasing temperatures $\tau=1.0, 0.5, 0.1$.
When $\tau$ is large, all $\phi_i$ will be pushed to neutral state, e.g. in figure.~\ref{fig:ex} (b) where the difference of $\phi_i$ is at scale $10^{-3}$. In this case, the noise $\zeta$ can significantly affect the sign of the gradient $g$ and lead to phase transitions. By gradually decreasing the temperature, $\phi$ collapses to the global optimum and provides correct guidance to update $\theta$. 

\begin{figure}
    \centering
    \includegraphics[width=\textwidth]{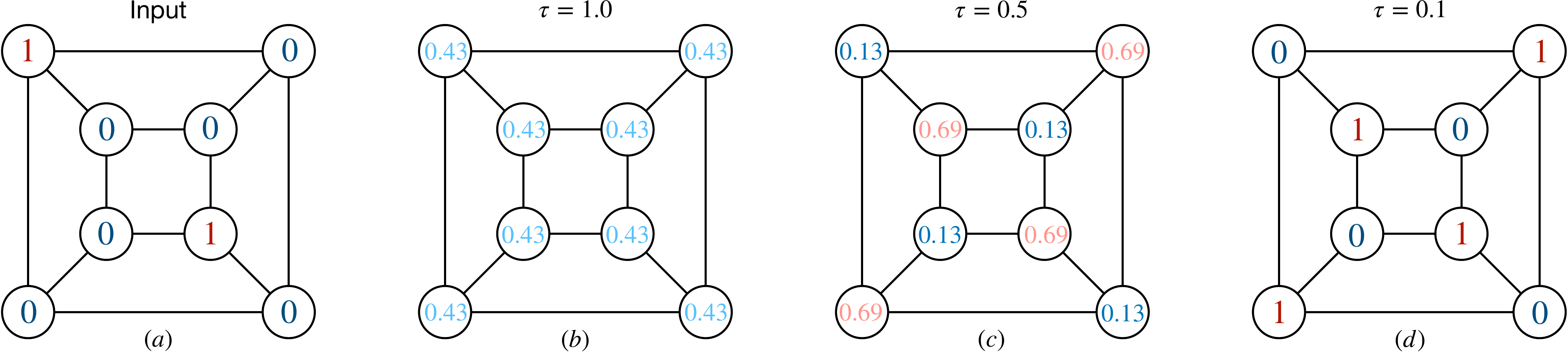}
    \caption{A toy example of maximum independent set}
    \label{fig:ex}
\end{figure}

\section{Case Study}
\label{sec:cases}
We consider four types of combinatorial optimization problems on graphs in this work: maximum independent set, maximum clique, minimum dominate set, and minimum cut. All problems can be represented by an undirected weighted graph $G=(V, E, w)$, where $V=\{1, ..., n\}$ is the set of nodes, $E$ is the set of edges, and $w$ is the weight function. For any $i \in V$, $w_i = w(i)$ is the weight of the node. For any $(i, j) \in E$, $w_{ij} = w(i, j)$ is the weight of the edge. For each problem, we derive the minimum value of the penalty coefficient $\beta$ such that the energy function has the lowest energy at optimal solutions, and we use the derived values to design the loss functions in our experiments.

\subsection{Maximum Independent Set}
An independent set is a subset of the vertices $S \subseteq V$, such that for arbitrary $i, j \in S$, $(i, j) \notin E$. The maximum independent set problem is to find an independent set $S$ having the largest weight. Rigorously, if we denote $x_i = 1$ to indicate $i \in S$ and $x_i = 0$ to indicate $i \notin S$, the problem can be formulated as:
\begin{equation}
    \argmin_{x \in \{0, 1\}^n} c(x) := - \sum_{i=1}^n w_i x_i, \quad \text{ subject to } x_i x_j = 0, \forall (i, j) \in E \label{eq:is}
\end{equation}
We define the corresponding energy function:
\begin{equation}
    f(x) := -\sum_{i=1}^n w_i x_i + \sum_{(i, j) \in E} \beta_{ij} x_i x_j \label{eq:f_is}
\end{equation}
\begin{proposition}
\label{prop:mis}
If $\beta_{ij} \ge \max\{w_i, w_j\}$ for all $(i, j) \in E$, then for any $x \in \{0, 1\}^n$, there exists a $x' \in \{0, 1\}^n$ that satisfies the constraints in \eqref{eq:is} and has lower energy: $f(x') \le f(x)$.
\end{proposition}

\subsection{Maximum Clique}
A clique is a subset of the vertices $S \subseteq V$, such that every two distinct $i, j \in S$ are adjacent: $(i, j) \in E$. The maximum clique problem is to find a clique $S$ having the largest weight. Rigorously, if we denote $x_i = 1$ to indicate $i \in S$ and $x_i = 0$ to indicate $i \notin S$, the problem can be formulated as:
\begin{equation}
    \argmin_{x\in \{0, 1\}^n} c(x) := - \sum_{i=1}^n w_i x_i, \quad \text{ subject to } x_i x_j = 0, \forall (i, j) \in E^c \label{eq:mc}
\end{equation}
where $E^c = \{(i, j) \in V \times V: i \neq j, (i, j) \notin E\}$ is the set of complement edges on graph $G$. We define the corresponding energy function:
\begin{equation}
    f(x) := - \sum_{i=1}^n w_i x_i + \sum_{(i, j) \in E^c} \beta_{ij} x_i x_j \label{eq:f_mc}
\end{equation}
\begin{proposition}
If $\beta_{ij} \ge \max\{w_i, w_j\}$ for all $(i, j) \in E^c$, then for any $x \in \{0, 1\}^n$, there exists a $x' \in \{0, 1\}^n$ that satisfies the constraints in \eqref{eq:mc} and has lower energy: $f(x') \le f(x)$.
\end{proposition}

\subsection{Minimum Dominate Set}
A dominate set is a subset of the vertices $S \subseteq V$, where for any $v \in V$, there exists $u \in S$ such that $(u, v) \in E$. The minimum dominate set problem is to find a dominate set $S$ having the minimum weight. Rigorously, if we denote $x_i = 1$ to indicate $i \in S$ and $x_i=0$ to indicate $i \notin S$, the problem can be formulated as:
\begin{equation}
    \argmin_{x\in \{0, 1\}^n} c(x) := \sum_{i=1}^n w_i x_i, \quad \text{ subject to } (1 - x_i) \prod_{j \in N(i)} (1 - x_j) =0, \forall i \in V \label{eq:mds}
\end{equation}
We define the corresponding energy function:
\begin{equation}
    f(x):= - \sum_{i=1}^n w_i x_i + \sum_{i=1}^n \beta_i (1-x_i) \prod_{j \in N(i)} (1 - x_j)
\end{equation}
\begin{proposition}
If $\beta_i \ge \min_k \{ w_k: k \in N(i) \text{ or } k = i\}$, then for any $x \in \{0, 1\}^n$, there exists a $x' \in \{0, 1\}^n$ that satisfies the constraints in \eqref{eq:mc} and has lower energy: $f(x') \le f(x)$.
\end{proposition}

\subsection{Minimum Cut}
A partition consists of two subsets: $S$ and $V \backslash S$. The cut $\text{cut}(S)$ is defined as the number of weights between $S$ and $V \backslash S$. The volume of $S$ is defined as $\text{vol}(S) = \sum_{i\in S} d_i$, where $d_i$ is the degree of node $i$. The minimum cut problem is to find a $S$ having the minimum cut, subject to the degree of $S$ is between $[D_0, D_1]$. Rigorously, if we denote $x_i = 1$ to indicate $i \in S$ and $x_i = 0$ to indicate $i \notin S$, the problem can be formulated as:
\begin{equation}
    \argmin_{x \in \{0, 1\}^n} c(x) := \sum_{(i, j)\in E} x_i (1 - x_j) w_{ij}, \quad \text{ subject to } D_0 \le \sum_{i=1}^n d_i x_i \le D_1
\end{equation}
We define the corresponding energy function:
\begin{equation}
    f(x) := \sum_{(i, j)\in E} x_i (1 - x_j) w_{ij} + \beta (\sum_{i=1}^n d_i x_i - D_1)_+ + \beta (D_0 - \sum_{i=1}^n d_i x_i )_+
\end{equation}
\begin{proposition}
If $\beta \ge \max_i\{\sum_{j \in N(i)} |w_{i, j}| \}$, then any $x \in \{0, 1\}^n$, there exists a $x' \in \{0, 1\}^n$ that satisfies the constraints in \eqref{eq:mc} and has lower energy: $f(x') \le f(x)$.
\end{proposition}

\section{Related Work}
Recently, there has been a surge of interest in learning algorithms for CO problems \citep{bengio2020machine}. Supervised learning is widely used. Numerous works have combined GNNs with search procedures to solve classical CO problems, such as the traveling salesman problem \citep{vinyals2015pointer, joshi2019efficient, prates2019learning}, graph matching \citep{wang2019learning, wang2020combinatorial}, quadratic assignments \citep{nowak2017note}, graph coloring \citep{lemos2019graph}, and maximum independent set \citep{li2018combinatorial}. Another fruitful direction is combining learning with existing solvers. For example, in the branch and bound algorithm, \citet{he2014learning, khalil2016learning, gasse2019exact, nair2020solving}  learn the variable selection policy by imitating the decision of oracle or rules designed by human experts. However, the success of supervised learning relies on large datasets with already solved instances, which is hard to efficiently generate in an unbiased and representative manner \citep{yehuda2020s}, 

Many works, therefore, choose to use reinforcement learning instead. \citet{dai2017learning} combines Q-learning with greedy algorithms to solve CO problems on graphs. Q-learning is also used in \citep{bai2020fast} for maximum subgraph problem. \citet{sun2020improving} uses an evolutionary strategy to learn variable selection in the branch and bound algorithm. \citet{yolcu2019learning} employs REINFORCE algorithm to learn local heuristics for SAT problems. \citet{chen2019learning} uses actor-critic learning to learn a local rewriting algorithm. Despite being a promising approach that avoids using labeled data, reinforcement learning is typically sample inefficient and notoriously unstable to train due to poor gradient estimations, correlations present in the sequence of observations, and hard explorations \citep{espeholt2018impala, tang2017exploration}.

Works in unsupervised learning show promising results. \citet{yao2019experimental} train GNN for the max-cut problem by optimizing a relaxation of the cut objective, \citet{toenshoff2021graph} trains RNN for maximum-SAT via maximizing the probability of its prediction. \citet{karalias2020erdos} use a graph neural network to predict the distribution and the graphical neural network to minimize the expectation of the objective function on this distribution. The probabilistic method provides a good framework for unsupervised learning. However, optimizing the distribution is typically non-convex \citep{mezard2009information}, making the training of the graph neural network very unstable. 

\section{Experiments}
\begin{table*}[tb]
\scriptsize
\caption{Evaluation of Maximum Independent Set} 
\label{tb:mis}
\begin{center}
    \begin{tabular*}{\textwidth}{l @{\extracolsep{\fill}} clclclcl}
    \toprule
    Size & \multicolumn{2}{c}{small} & \multicolumn{2}{c}{large} & \multicolumn{2}{c}{Collab} & \multicolumn{2}{c}{Twitter} \\
    \cmidrule(lr){1-1} \cmidrule(lr){2-3} \cmidrule(lr){4-5} \cmidrule(lr){6-7} \cmidrule(lr){8-9}  
    Method & \textit{ratio} & time (s) & \textit{ratio} & time (s) & \textit{ratio} & time (s) & \textit{ratio} & time (s) \\
    \midrule
    Erdos &  $0.805\pm0.052$ & $0.156$  & $0.781\pm0.644$ & $2.158$ & $0.986\pm0.056$ & $0.010$ & $0.975\pm0.033$ & $0.020$\\
    Our's  &  $\mathbf{0.898}\pm0.030$  & $0.165$ & $\mathbf{0.848}\pm0.529$ &  $2.045$ & $\mathbf{0.997}\pm0.020$ & $0.010$ & $\mathbf{0.986}\pm0.012$ & $0.020$\\
    \midrule
    Greedy & $0.761\pm0.058$ & $0.002$  & $0.720\pm0.046$ & $0.009$ & $0.996\pm0.017$ & $0.001$ & $0.957\pm0.037$ & $0.006$  \\
    MFA  & $0.784 \pm 0.058$ & $0.042$  & $0.747\pm 0.056$ & $0.637$  & $0.998\pm0.007$ & $0.002$ & $0.994\pm0.010$ & $0.003$  \\
    RUNCSP    &  $0.823 \pm 0.145$ & $1.936$  & $0.587\pm0.312$ & $7.282$ & $0.912\pm0.101$ & $0.254$ & $0.845\pm0.184$ &  $4.429$   \\
    G(0.5s) & $0.864\pm0.169$ & $0.723$  & $0.632\pm0.176$ &  $1.199$ & $1.000\pm0.000$ & $0.029$ & $0.950\pm0.191$ & $0.441$  \\
    G(1.0s) &  $0.972\pm0.065$ & $1.063$  & $0.635\pm0.176$  &   $1.686$ & $1.000\pm0.000$& $0.029$ &  $1.000\pm0.000$ & $0.462$ \\
    \bottomrule
    \end{tabular*}
\end{center}
\end{table*}

\begin{table*}[tb]
\scriptsize
\caption{Evaluation of Maximum Clique} 
\label{tb:mc}
\begin{center}
    \begin{tabular*}{\textwidth}{l @{\extracolsep{\fill}} clclclcl}
    \toprule
    Size & \multicolumn{2}{c}{small} & \multicolumn{2}{c}{large} & \multicolumn{2}{c}{Collab} & \multicolumn{2}{c}{Twitter} \\
    \cmidrule(lr){1-1} \cmidrule(lr){2-3} \cmidrule(lr){4-5} \cmidrule(lr){6-7} \cmidrule(lr){8-9}  
    Method & \textit{ratio} & time (s) & \textit{ratio} & time (s) & \textit{ratio} & time (s) & \textit{ratio} & time (s) \\
    \midrule
    Erdos & $0.813\pm0.067$  & $0.279$  & $0.735\pm0.084$ & $0.622$  & $0.960\pm0.019$ & $0.139$ &  $0.822\pm0.085$&  $0.222$ \\
    Our's  & $\mathbf{0.901}\pm0.055$ &  $0.262$ & $\mathbf{0.831}\pm0.078$ & $0.594$ & $\mathbf{0.988}\pm0.011$ & $0.143$ &  $\mathbf{0.920}\pm0.083$ &$0.213$\\
    \midrule
    Greedy   & $0.764\pm0.064$ & $0.002$  & $0.727\pm0.038$   & $0.014$ & $0.999\pm0.002$ & $0.001$ & $0.959\pm0.034$ & $0.001$ \\
    MFA      & $0.804\pm0.064$ & $0.144$ & $0.710\pm0.045$ & $0.147$ & $1.000\pm0.000$ & $0.005$ & $0.994\pm0.010$ & $0.010$ \\
    RUNCSP   &  $0.821 \pm 0.131$ & $2.045$  & $0.574\pm0.299$ & $7.332$ & $0.887\pm0.134$ & $0.164$ & $0.832\pm0.153$ &  $4.373$   \\
    G(0.5s) & $0.948\pm0.076$  & $0.599$  & $0.812\pm0.087$  &  $0.617$ & $ 0.997\pm0.035$ & $0.061$ & $0.976\pm0.065$ & $0.382$   \\
    G(1.0s) & $0.984\pm0.042$ & $0.705$ & $0.847\pm0.101$ & $1.077$ & $0.999\pm0.015$& $0.062$ &  $0.997\pm0.029$ & $0.464$    \\

    \bottomrule
    \end{tabular*}
\end{center}
\end{table*}
\subsection{Settings}
{\bf Dataset}:
For maximum independent set and maximum clique, problems on both real graphs and random graphs are easy \citep{dai2020framework}. Hence, we follow 
\citet{karalias2020erdos} to use RB graphs \citep{xu2007random}, which has been specifically designed to generate hard instances. We use a small size dataset containing graphs with 200-300 nodes and a large size dataset containing graphs with 800-1200 nodes. For maximum dominate set, we follow \citet{dai2020framework} to use BA graphs with 4 attaching edges \citep{barabasi1999emergence}. We also use a small size dataset containing graphs with 200-300 nodes and a large size dataset containing graphs with 800-1200 nodes. We also use real graph datasets Collab, Twitter from TUdataset \citep{morris2020tudataset}. For minimum cut, we follow \citet{karalias2020erdos} and use real graph datasets including SF-295 \citep{yan2008mining}, Facebook \citep{traud2012social}, and Twitter \citep{morris2020tudataset}. For RB graphs, the optimal solution is known in generation. For other problems, we generate the "ground truth" solution through Gurobi 9.5 \citep{gurobi19gurobi} with a time limit of 3600 seconds. For synthetic datasets, we generate 2000 graphs for training, 500 graphs for validation, and 500 graphs for testing. For real datasets, we follow \citet{karalias2020erdos} and use a 60-20-20 split for training, validating and testing.

{\bf Implementation}:
We train our graph neural network on training data with 500 epochs.
We choose the penalty coefficient $\beta$ at the critical point for each problem type. We use the schedule:
\begin{equation}
    \tau_k = \tau_0 / (1 + \alpha k) \label{eq:linear_schedule}
\end{equation}
where $\tau_0$ is chosen as the Lipschitz constant of the energy function \eqref{eq:energy} and $\alpha$ is selected to make sure the final temperature $\tau_500 = 0.001$.
Since the contribution of this work focuses on the training framework, the architecture of the graph neural network is not important. Hence, we simply follow \citet{karalias2020erdos} for fair comparison. In particular, the architecture consists of multiple layers of the Graph Isomorphism Network \citep{xu2018powerful} and a graph Attention \citep{velivckovic2017graph}. Each convolutional layer was equipped with skip connections, batch normalization, and graph size normalization \citep{dwivedi2020benchmarking}. More details refer to \citet{karalias2020erdos}. After obtaining the variational distribution $Q_\phi$ \eqref{eq:q}, we generate the solution via conditional decoding \citep{raghavan1988probabilistic}.

{\bf Baselines}: We compare our method with unsupervised neural methods, classical algorithms, and integer programming solvers. To establish a strong baseline for neural methods, we use the Erdos method \citep{karalias2020erdos}, the state-of-the-art unsupervised learning framework for combinatorial optimization problems. For maximum clique and maximum independent set, we transform them to maximum 2-sat and compare with RUNCSP \citep{toenshoff2021graph}. For minimum cut, we follow \citet{karalias2020erdos} and built the L1 GNN and L2 GNN. In classical algorithms, we consider greedy algorithms and mean field annealing (MFA) \citep{bilbro1988optimization}. MFA also runs mean field approximation \citep{anderson1988neural} to predict a variational distribution as our method. The difference is that the update rule of MFA is determined after seeing the current graph, while the parameters in GNN is trained on the whole dataset. Also, in minimum cut, we follow \citep{karalias2020erdos} to compare with well-known and advanced algorithms: \textit{Pageran-Nibble} \citep{andersen2006local}, Captacity Releasing Diffusion (CRD) \citep{wang2017capacity}, Max-flow Quotient-cut Improvement \citep{lang2004flow}, and Simple-Local \citep{veldt2016simple}. For integer programming solver, we use Gurobi 9.0 \citep{gurobi19gurobi} and set different time limits $t$. We denote G($t$s) as Gurobi 9.0 ($t$ s). where $t$ is the time limit.

\vspace{-2mm}
\subsection{Results}
\vspace{-1mm}
We report the results for maximum independent set in Table \ref{tb:mis}, the results for maximum clique in Table \ref{tb:mc}, the results for minimum dominate set in Table \ref{tb:mds}, the results for minimum cut in Table \ref{tb:cut}. In maximum independent set and maximum clique, we report the ratios computed by dividing the optimal value by the obtained value (the larger, the better). In minimum dominating set, we report the ratios computed from the obtained value by dividing the optimal value (the larger, the better). In minimum cut, we follow \citet{karalias2020erdos} and evaluate the performance via local conductance: $\text{cut}(S) / \text{vol}(S)$ (the smaller the better). We can see that the annealed training substantially improves the performance of Erdos across all problem types and all datasets by utilizing a better unsupervised training framework. Besides, with annealed training, the learned GNN outperforms meanfield annealing in most problems, which indicates learning the shared patterns in graphs is helpful.

\vspace{-2mm}
\subsection{Parameter Change Distance}
\vspace{-1mm}
We want to stress that we use precisely the same graph neural network as Erdos, and the performance improvements come from our annealed training framework. In scatter plot \ref{fig:distance_mis}, \ref{fig:distance_mds}, we report the relative change for the parameters of GNN in maximum independent set and minimum dominate set problems on the Twitter dataset. The relative change is calculated as $\frac{\|u - v\|_2}{\|v\|_2}$, where $v$ and $u$ are vectors flattened from the parameters of GNN before and after training. For each method, we run 20 seeds. After introducing the annealed training, we can see that both the ratio and the relative change of the parameters have systematic increase, which means the parameters of GNN can traverse to more distant regions and find better optima in annealed learning. We believe these evidence effectively support that annealed training prevents the training from being stuck at local optima.

\begin{table*}[tb]
\scriptsize
\caption{Evaluation of Minimum Dominate Set} 
\label{tb:mds}
\begin{center}
    \begin{tabular*}{\textwidth}{l @{\extracolsep{\fill}} clclclcl}
    \toprule
    Size & \multicolumn{2}{c}{small} & \multicolumn{2}{c}{large} & \multicolumn{2}{c}{Collab} & \multicolumn{2}{c}{Twitter} \\
    \cmidrule(lr){1-1} \cmidrule(lr){2-3} \cmidrule(lr){4-5} \cmidrule(lr){6-7} \cmidrule(lr){8-9}  
    Method & \textit{ratio} & time (s) & \textit{ratio} & time (s) & \textit{ratio} & time (s) & \textit{ratio} & time (s) \\
    \midrule
    Erdos & $0.909\pm0.037$ & $0.121$ & $0.889\pm0.017$ & $0.449$  & $0.982\pm0.070$ & $0.007$  & $0.924\pm 0.098$ & $0.015$ \\
    Our's  & $\mathbf{0.954}\pm0.006$ & $0.120$ & $\mathbf{0.931}\pm0.015$ & $0.453$ & $\mathbf{0.993}\pm0.062$ & $0.006$ & $\mathbf{0.952}\pm0.074$ & $0.016$  \\
    \midrule
    Greedy   & $0.743\pm0.053$ & $0.254$ & $0.735\pm0.026$ &  $3.130$ &  $0.661\pm0.406$ & $0.028$  & $0.741\pm0.142$ & $0.079$ \\
    MFA      & $0.926\pm0.032$ & $0.213$ & $0.910\pm0.016$ & $3.520$ &  $0.895\pm0.210$ & $0.030$ & $0.952\pm0.076$ & $0.099$  \\
    G(0.5s) & $0.993\pm0.014$ & $0.381$ & $0.994\pm0.013$ &  $0.384$ &$1.000\pm0.000$ & $0.042$ & $1.000\pm0.000$& $0.084$ \\
    G(1.0s) & $0.999\pm0.005$ & $0.538$ & $0.999\pm0.005$  &  $0.563$ &$1.000\pm0.000$ & $0.042$ & $0.839\pm0.000$ & $0.084$  \\
    \bottomrule
    \end{tabular*}
\end{center}
\end{table*}

\begin{table*}[tb]
\scriptsize
\caption{Evaluation of Minimum Cut} 
\label{tb:cut}
\begin{center}
    \begin{tabular*}{\textwidth}{l @{\extracolsep{\fill}} crcrcr}
    \toprule
    Size & \multicolumn{2}{c}{SF-295} & \multicolumn{2}{c}{Facebook} & \multicolumn{2}{c}{Twitter} \\
    \cmidrule(lr){1-1} \cmidrule(lr){2-3} \cmidrule(lr){4-5} \cmidrule(lr){6-7}
    Method & \textit{ratio} & time (s) & \textit{ratio} & time (s) & \textit{ratio} & time (s)  \\
    \midrule
    Erdos & $\mathbf{0.124}\pm0.001$  & $0.22$  & $0.156\pm0.026$  & $289.3$ & $0.292\pm0.009$ & $6.17$   \\
    Our's  & $0.135\pm0.011$  & $0.23$  & $\mathbf{0.151}\pm0.045$ & $290.5$ & $\mathbf{0.201}\pm0.007$ & $6.16$  \\
    \midrule
    L1 GNN   & $0.188\pm0.045$ & $0.02$ & $0.571\pm0.191$ & $13.83$  &$0.318\pm0.077$ & $0.53$ \\
    L2 GNN  & $0.149\pm0.038$ & $0.01$  & $0.305\pm 0.082$ & $13.83$ & $0.388 \pm 0.074$ & $0.53$  \\
    Pagerank-Nibble & $0.375\pm0.001$ & $1.48$  &  N/A & N/A  &  $0.603\pm0.005$ & $20.62$  \\
    CRD & $0.364\pm0.001$ & $0.03$ & $0.301\pm0.097$ & $596.46$ & $0.502\pm0.020$  & $20.35$   \\
    MQI & $0.659\pm0.000$ & $0.03$ & $0.935\pm0.024$ & $408.52$ & $0.887 \pm 0.007$ & $0.71$\\
    Simple-Local & $0.650\pm0.024$ & $0.05$ & $0.961\pm0.019$ & $1787.79$ & $0.895 \pm 0.006$ & $0.84$ \\
    G(10s) & $0.105\pm0.000$ & $0.16$ & $0.961\pm0.010$ & $1787.79$ & $0.535\pm0.006$ & $52.98$  \\
    \bottomrule
    \end{tabular*}
\end{center}
\end{table*}

\vspace{-2mm}
\section{Ablation Study}
\vspace{-1mm}
We conduct an ablation study to answer two questions: 
\begin{enumerate}
\item How does the penalty coefficient $\beta$ in \eqref{eq:energy} influence the performance?
\item How does the annealing schedule influence the performance?
\end{enumerate}
 We conduct the experiments for the minimum dominating set problem on the small BA graphs from the previous section.

\vspace{-2mm}
\subsection{Penalty Coefficient}
\vspace{-1mm}
In minimum dominating set problem, we know that the minimum penalty coefficient $\beta$ needed to make sure the EBMs unbiased on the unweighted BA graphs is $\beta = 1.0$. To justify the importance to use the minimum penalty, we evaluate the performance for $\beta =$ \{0.0, 0.25, 0.5, 0.75, 1.0, 2.0, 3.0, 5.0\}. For each $\beta$, we run experiments with five random seeds, and we report the result in Figure \ref{fig:ablation_beta}. We can see that the minimum penalty $\beta = 1$ has the best ratio. When the penalty coefficient $\beta < 1$, the EBMs \eqref{eq:model} are biased and have weights on infeasible solutions, thereby reduces the performance. When the penalty coefficient $\beta > 1$, the energy model \eqref{eq:model} becomes less smooth and increases the difficulty in training. The penalty coefficient $\beta=1$ gives the smoothest unbiased EBMs and has the best performance. We want to note that, when $\beta = 0$, the loss function is non-informative, and the performance ratio can be as low as $0.3$, so we do not plot its result in the figure.
\begin{figure}
    \centering
    \begin{minipage}{0.32\textwidth}
    \centering
    \includegraphics[width=.99\textwidth]{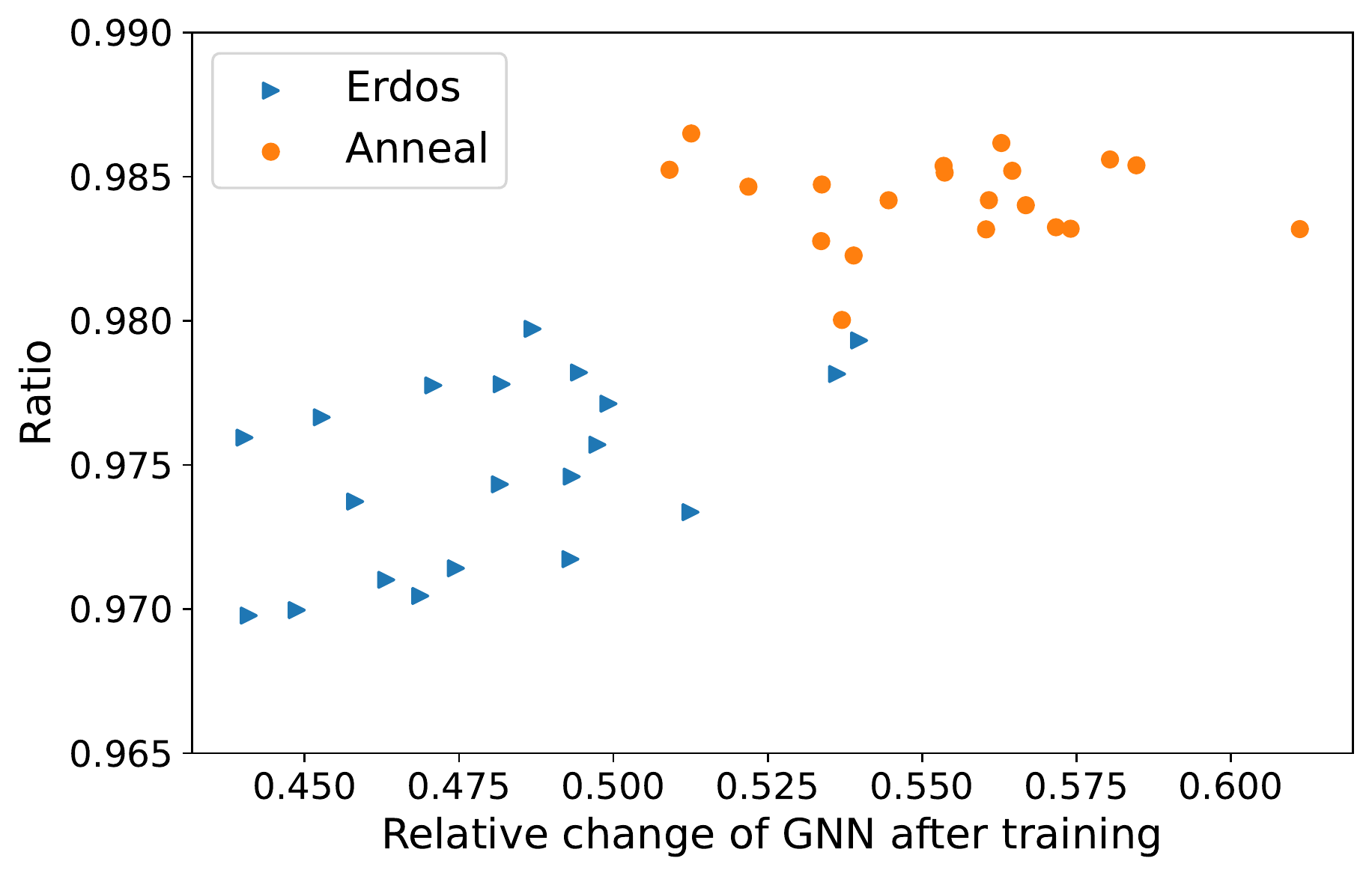}
    \caption{Distance in MIS}
    \label{fig:distance_mis}
    \end{minipage}
    \begin{minipage}{0.32\textwidth}
    \centering
    \includegraphics[width=.99\textwidth]{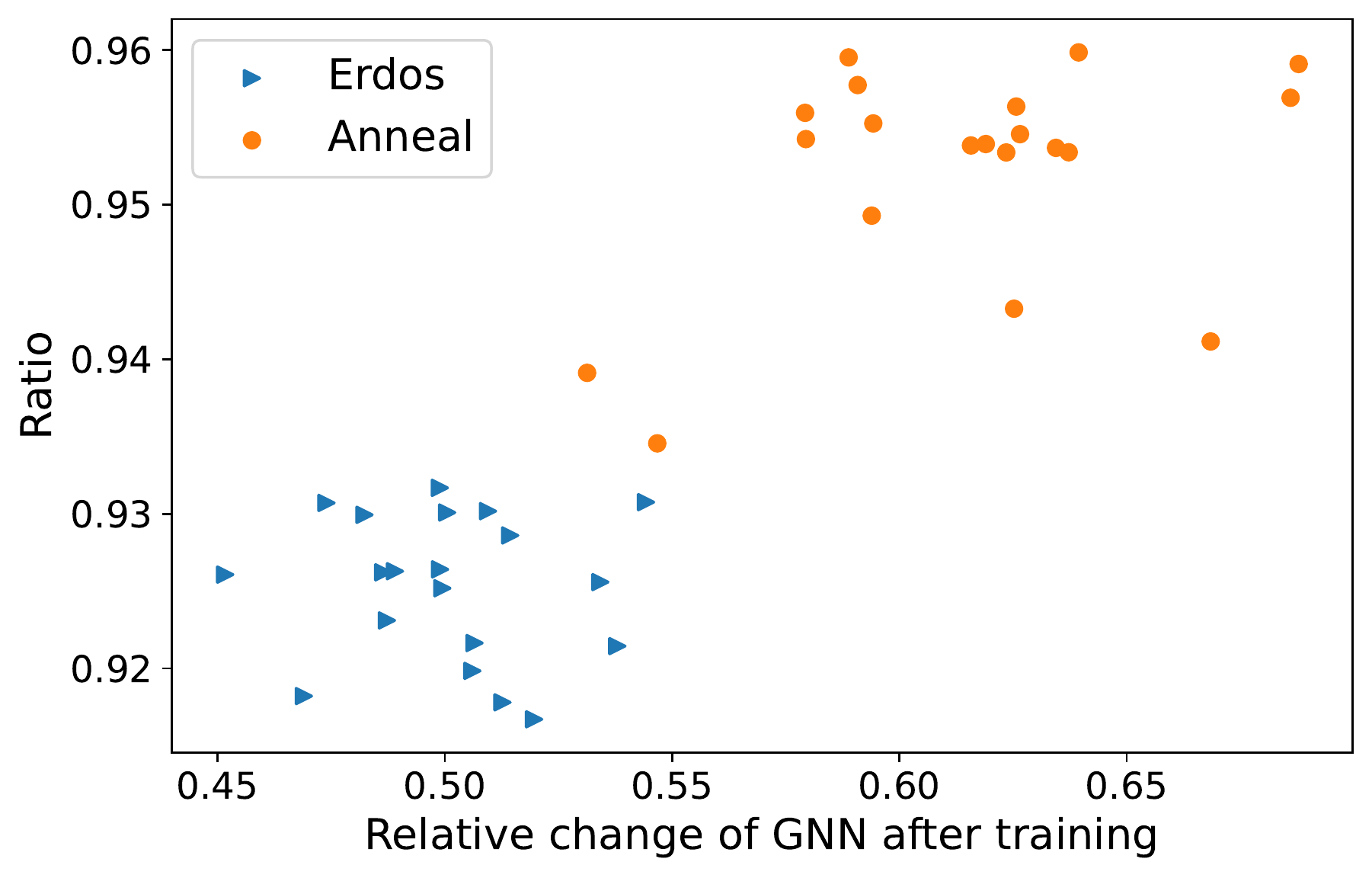}
    \caption{Distance in MDS}
    \label{fig:distance_mds}
    \end{minipage}
    \begin{minipage}{0.32\textwidth}
    \centering
    \includegraphics[width=.99\textwidth]{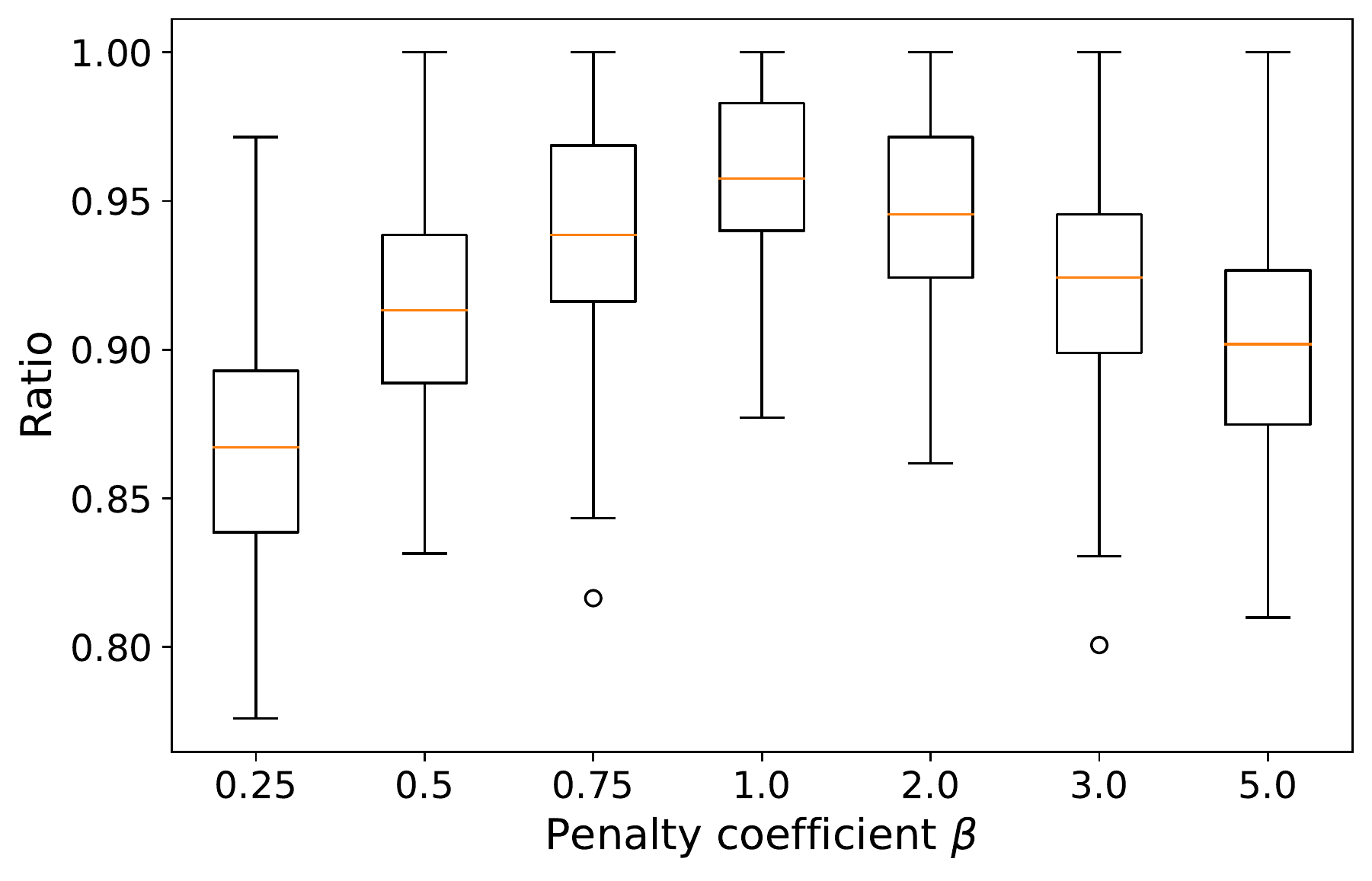}
    \caption{Ablation for $\beta$}
    \label{fig:ablation_beta}
    \end{minipage}
\end{figure}

\vspace{-2mm}
\subsection{Annealing Schedule}
\vspace{-1mm}
We use the schedule \eqref{eq:linear_schedule} so as to make sure the potential change $f / \tau_{k+1} - f / \tau_{k} \equiv C$ is a constant for all steps $k$. In fact, with the schedule \eqref{eq:linear_schedule}, the potential $f / \tau_k = (1 + \alpha (k-1)) f / \tau_0$ is a linear function w.r.t. $k$. Hence, we name it as linear schedule. It is possible to use other schedules, e.g. $f / \tau_k = (1 + \alpha (k-1))^\frac{1}{2} f / \tau_0$ and $f / \tau_k = (1 + \alpha (k-1))^3 f / \tau_0$, and we name them as concave and convex schedule. The visualization of the temperature schedule and the potential schedule is given in Figure \ref{fig:ablation_schedule}. Besides schedule, the initial temperature is also an important hyperparameter. We evaluate the initial temperature $\tau_0 = $ \{0.0, 0.1, 0.5, 1.0, 2.0, 5.0\}. We report the results in Figure \ref{fig:ablation_schedule}. We can see that the performance is robust for whatever convex, linear, or concave schedule used. The more important factor is the initial temperature $\tau_0$. The performance is reduced when $\tau_0$ is too small as the energy based model \eqref{eq:model} is not smooth enough and the performance is robust when $\tau_0$ is large.

\begin{figure}
    \centering
    \includegraphics[width=.9\textwidth]{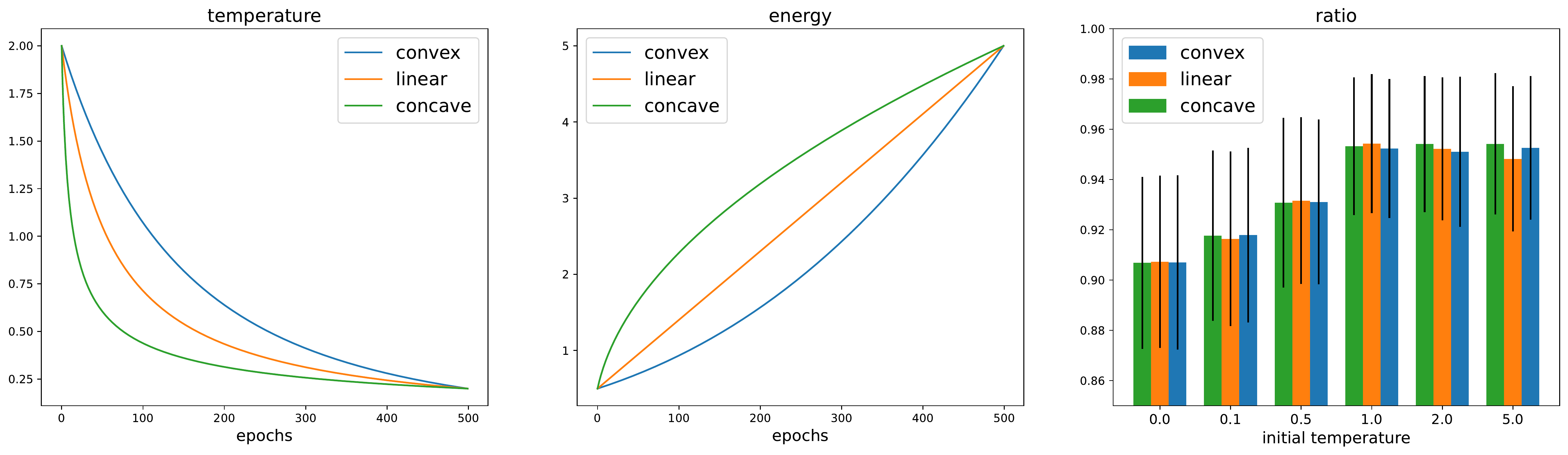}
    \caption{Ablation for annealing schedule}
    \label{fig:ablation_schedule}
\end{figure}

\vspace{-2mm}
\section{Discussion}
\vspace{-1mm}
This paper proposes a generic unsupervised learning framework for combinatorial optimization problems and substantially improve the performance of the state-of-the-art method. The success of the framework relies on smoothing the loss function via critical penalty coefficients and annealed training as they effectively prevent the training from being stuck at local optima. The techniques introduced here can be potentially applied in a broader context beyond combinatorial optimization, especially in the weakly supervised learning setting like logic reasoning~\citep{huang2021scallop}, program induction~\citep{chen2020compositional}, question answering~\citep{ren2021lego} where fine-grained supervisions are missing and required to be inferred.

\bibliography{reference}
\bibliographystyle{icml2022}

\appendix

\section{Complete Proof}
\label{app:proof}
\subsection{Maximum Independent Set}
In maximum independent set, we use the energy function:
\begin{equation}
    f(x) := -\sum_{i=1}^n w_i x_i + \sum_{(i, j) \in E} \beta_{ij} x_i x_j
\end{equation}
We are going to prove the following proposition.
\begin{proposition}
If $\beta_{ij} \ge \min\{w_i, w_j\}$ for all $(i, j) \in E$, then for any $x \in \{0, 1\}^n$, there exists a $x' \in \{0, 1\}^n$ that satisfies the constraints in \eqref{eq:is} and has lower energy: $f(x') \le f(x)$.
\end{proposition}
\begin{proof}
For arbitrary $x \in \{0, 1\}^n$, if $x$ satisfies all constraints, we only need to let $x' = x$. Else, there must exist an edge $(i, j) \in E$, such that $x_i x_j = 1$. Denote $k = \argmin\{w_i, w_j\}$, we define $x'_i = x_i$ if $i \neq k$ and $x'_k = 0$. In this case, we have:
\begin{equation}
    f(x') - f(x) = w_k - \sum_{j \in N(k)} \beta_{k,j} x_j \le w_k (1 - \sum_{j \in N(k)} x_j) \le 0
\end{equation}
Thus we show $f(x') \le f(x)$. 

On the other side, consider a graph $G = (V=\{1, 2\}, E=\{(1, 2)\})$ and $\beta_{12} < w_1 < w_2$. Then the maximum independent set is $\{2\}$, which can be represented by $x=(0, 1)$. However, in this case, let $x'=(1, 1)$ is feasible while $f(x') \le f(x)$. This means the condition we just derived is sharp.
\end{proof}

\subsection{Maximum Clique}
In maximum independent set, we use the energy function:
\begin{equation}
    f(x) := - \sum_{i=1}^n w_i x_i + \sum_{(i, j) \in E^c} \beta_{ij} x_i x_j
\end{equation}
We are going to prove the following proposition.
\begin{proposition}
If $\beta_{ij} \ge \min\{w_i, w_j\}$ for all $(i, j) \in E^c$, then for any $x \in \{0, 1\}^n$, there exists a $x' \in \{0, 1\}^n$ that satisfies the constraints in \eqref{eq:mc} and has lower energy: $f(x') \le f(x)$.
\end{proposition}
\begin{proof}
For arbitrary $x \in \{0, 1\}^n$, if $x$ satisfies all constraints, we only need to let $x' = x$. Else, there must exist an edge $(i, j) \in E^c$, such that $x_i x_j = 1$. Denote $k = \argmin\{w_i, w_j\}$, we define $x'_i = x_i$ if $i \neq k$ and $x'_k = 0$. In this case, we have:
\begin{equation}
    f(x') - f(x) = w_k - \sum_{j: (k, j) \in E^c} \beta_{k,j} x_j \le w_k (1 - \sum_{j: (k, j) \in E^c} x_j) \le 0
\end{equation}
Thus we show $f(x') \le f(x)$. 

On the other side, consider a graph $G = (V=\{1, 2\}, E=\{\empty\})$ and $\beta_{12} < w_1 < w_2$. Then the maximum clique is $\{2\}$, which can be represented by $x=(0, 1)$. However, in this case, let $x'=(1, 1)$ is feasible while $f(x') \le f(x)$. This means the condition we just derived is sharp.
\end{proof}

\subsection{Minimum Dominate Set}
In maximum independent set, we use the energy function:
\begin{equation}
    f(x):= \sum_{i=1}^n w_i x_i + \sum_{i=1}^n \beta_i (1-x_i) \prod_{j \in N(i)} (1 - x_j)
\end{equation}
We are going to prove the following proposition.
\begin{proposition}
If $\beta_i \ge \min_k \{ w_k: k \in N(i) \text{ or } k = i\}$, then for any $x \in \{0, 1\}^n$, there exists a $x' \in \{0, 1\}^n$ that satisfies the constraints in \eqref{eq:mc} and has lower energy: $f(x') \le f(x)$.
\end{proposition}
\begin{proof}
For arbitrary $x \in \{0, 1\}^n$, if $x$ satisfies all constraints, we only need to let $x' = x$. Else, there must exist a node $t \in V$, such that $x_t = 0$ and $x_j = 0$ for all $j \in N(t)$. Let $k = \argmin\{w_j: j \in N(t), \text{ or } j = t\}$, we define $x'_i = x_i$ if $i \neq k$ and $x'_k = 1$. In this case, we have:
\begin{equation}
    f(x') - f(x) = w_k - \beta_t + \sum_{i\neq t} \beta_i \Big[(1-x'_i) \prod_{j \in N(i)} (1 - x'_j) - (1-x_i) \prod_{j \in N(i)} (1 - x_j)\Big] \le 0
\end{equation}
Thus, we prove $f(x') \le f(x)$. 

On the other side, consider a graph $G = (V=\{1\}, E=\{\empty\})$ and $\beta_{1} < w_1$. Then the maximum clique is $\{1\}$, which can be represented by $x=(1)$. However, in this case, let $x'=(0)$ is feasible while $f(x') \le f(x)$. This means the condition we just derived is sharp.
\end{proof}

\subsection{Minimum Cut}
In maximum independent set, we use the energy function:
\begin{equation}
    f(x) := \sum_{(i, j)\in E} x_i (1 - x_j) w_{ij} + \beta (\sum_{i=1}^n d_i x_i - D_1)_+ + \beta (D_0 - \sum_{i=1}^n d_i x_i )_+
\end{equation}
We are going to prove the following proposition.
\begin{proposition}
If $\beta \ge \max_i\{\sum_{j \in N(i)} |w_{i, j}| \}$, then any $x \in \{0, 1\}^n$, there exists a $x' \in \{0, 1\}^n$ that satisfies the constraints in \eqref{eq:mc} and has lower energy: $f(x') \le f(x)$.
\end{proposition}


\section{Experiment Details}
\label{app:experiment}
\subsection{Hardware}
All methods were run on Intel(R) Xeon(R) Gold 5215 CPU @ 2.50GHz, with 377GB of available RAM. The neural networks were executed on a single RTX6000 25GB graphics card. The code was executed on version 1.9.0 of PyTorch and version 1.7.2 of PyTorch Geometric.

\end{document}